\documentclass{article}
\usepackage{ijcai16}
\usepackage{amsthm}
\usepackage{amsmath}
\DeclareMathOperator{\Tr}{Tr}
\usepackage{times}

\usepackage{array}
\usepackage{url}
\usepackage{bbm}
\usepackage{amssymb}
\usepackage{amsmath}
\usepackage{amsfonts}
\usepackage{amsthm}
\usepackage[ruled,vlined,linesnumbered]{algorithm2e}
\usepackage{verbatim}
\usepackage{multirow}
\usepackage{makecell}
\usepackage{lscape}
\usepackage{color}
\usepackage{xcolor,colortbl}
\usepackage{listings}
\usepackage{tabularx}
\usepackage{graphicx} 
\usepackage{subfigure}
\usepackage{subfloat}
\usepackage{pgfplots}
\usepackage{lipsum}

\title{
Enforcing Template Representability and Temporal Consistency \\
for Adaptive Sparse Tracking
}
\author{Xue Yang, Fei Han, Hua Wang, and Hao Zhang
\thanks{Corresponding Author. This work was partially supported by the following grants: NSF-IIS 1423591.} \\
Department of Electrical Engineering and Computer Science\\
Colorado School of Mines, Golden, Colorado 80401\\
xueyang@mines.edu, fhan@mines.edu, huawangcs@gmail.com, hzhang@mines.edu
}

\begin{document}

\maketitle

\begin{abstract}
Sparse representation has been widely studied in visual tracking,
which has shown promising tracking performance.
Despite a lot of progress,
the visual tracking problem is still a challenging task due to appearance variations over time.
In this paper, we propose a novel sparse tracking algorithm
that well addresses temporal appearance changes, by
enforcing template representability and temporal consistency (TRAC).
By modeling temporal consistency,
our algorithm addresses the issue of drifting away from a tracking target.
By exploring the templates' long-term-short-term representability,
the proposed method adaptively updates the dictionary
 using the most descriptive templates,
which significantly improves the robustness to target appearance changes.
We compare our TRAC algorithm against the state-of-the-art approaches
on 12 challenging benchmark image sequences.
Both qualitative and quantitative results demonstrate that our algorithm significantly outperforms
previous state-of-the-art trackers.
\end{abstract}

\section{Introduction}
Visual tracking is one of the most important topics in computer vision
with a variety of applications such as surveillance, robotics, and motion analysis.
Over the years, numerous visual tracking methods have been proposed with demonstrated success \cite{yilmaz2006object,salti2012adaptive}.
However, tracking a target object under different circumstances robustly
remains a challenging task due to the challenges like occlusion, pose variation, background clutter, varying view point, illumination and scale change.
In recent years, sparse representation and particle filtering
have been widely studied to solve the visual tracking problem
\cite{mei2011robust,mei2011minimum}.
In this framework, particles are randomly sampled around the previous target state
according to Gaussian distributions,
each particle is sparsely represented by a dictionary of templates
and the particle with the smallest representation error is selected as the tracking result.
The sparse representation of each particle can be solved using $\ell_1$ minimization.
Multi-task learning improves the performance by solving all particles together as a multi-task problem
using mixed $\ell_{2,1}$ norm,
which can exploit the intrinsic relationship among all particles \cite{zhang2012robust}.
The sparse trackers have demonstrated robustness to image occlusion and lighting changes.
However, the temporal consistency of target appearances over time was not well investigated,
which is critical to track deformable/changing objects in cluttered environments.
In addition,
previous template update schemes based only on an importance weight
can result in a set of similar templates,
which limits the representability of the templates
and makes the trackers sensitive to appearance changes over time.

To make visual tracking robust to appearance changes like pose changes, rotation, and deformation,
we introduce a novel sparse tracking algorithm
that incorporates template representability and temporal consistency (TRAC).
Our contributions are threefold:
(1) We propose a novel method to model \emph{temporal consistency} of target appearances
in a short time period via sparsity-inducing norms,
which can well address the problem of tracker drifting.
(2) We introduce a novel \emph{adaptive template update} scheme
that considers the representability of the templates
beyond only using traditional important weights,
which significantly improves the templates' discriminative power.
(3) We develop a new optimization algorithm to efficiently solve the formulated problems,
with a theoretical guarantee to converge to the global optimal solution.

The reminder of the paper is organized as follows.
Related background is discussed in Section 2.
Our novel TRAC-based tracking is proposed in Section 3.
After showing experimental results in Section 4,
we conclude the paper in Section 5.

\section{Background}


\subsection{Related Work}

Visual tracking has been extensively studied over the last few decades.
Comprehensive surveys of tracking methods can be found in \cite{salti2012adaptive,smeulders2014visual}.
In general, existing tracking methods can be categorized as either discriminative or generative.
Discriminative tracking methods formulate the tracking problem as a binary classification task
that separates a target from the background.
\cite{babenko2009visual} proposed a multi instance learning algorithm that trained a discriminative classifier
in an online manner to separate the object from the background. \cite{kalal2010pn} used a bootstrapping binary classifier with positive and negative constraints for object tracking by detection. An online SVM solver was extended with latent variables in \cite{yao2013part} for structural learning of the tracking target.
Generative tracking techniques \cite{zhang2013real}, on the other hand, are based on appearance models of target objects
and search the most similar image region.
The appearance model can either rely on key points and finding correspondences on deformable objects \cite{nebehay2015clustering} or on image features extracted from a bounding box  \cite{zhang2013real}.
We focus on appearance models relying on image features,
which can be used to construct a descriptive representation of target objects.

Recently, sparse representation was introduced in generative tracking methods,
which demonstrated promising performance \cite{mei2011robust,liu2010robust,li2011real}.
In sparse trackers,
a candidate is represented by a sparse linear combination of target templates and trivial templates.
The trivial templates can handle occlusion by activating a limited number of trivial template coefficients, while the whole coefficients are sparse.
The sparse representation can be learned by solving an optimization problem regularized by sparsity-inducing norms.
Techniques using the $\ell_1$ norm regularization to build sparse representation models are often referred to as the L1 tracker.
\cite{bao2012real} improved the L1 tracker by adding an $\ell_2$ norm regularization on the trivial templates to increase tracking performance when no occlusion is present.
Considering the inherent low-rank structure of particle representations that can be learned jointly, \cite{zhang2012low} formulated the sparse representation problem as a low-rank matrix learning problem.
A multi-task learning was proposed to jointly learn the sparse representation of all particles
under this tracking framework based on particle filters \cite{zhang2012robust},
which imposed a joint sparsity using a mixed $\ell_{p,1}$ norm to encourage the sparseness of particles' representations that share only a few target templates.
Besides developing sparse representation models, many research focused on studying effective visual features
that can well distinguish the target from the background.
\cite{jia2012visual} proposed a local structural model that samples overlapped image patches within the target region to locate the target and handle partial occlusion.
 \cite{hong2013tracking}
utilized multiple types of features, including color, shape, and texture, in jointly sparse representations shared among all particles.
In \cite{zhang2015structural}, global and local features were imposed together with predefined spatial layouts considering the relationship among global and local appearance as well as the spatial structure of local patches.
Global and local sparse
representations were also developed in \cite{zhong2012robust},
using feature selection and a combination of generative and discriminative learning methods. 
However, the previous sparse trackers generally ignore the temporal
consistency of the target in a short history of frames,
which is addressed in this work.

For accurate visual tracking,
templates must be updated to account for target appearance changes and prevent drift problems.
Most of the sparse-based trackers adopted the template update scheme from the work in \cite{mei2011robust},
which assigns an importance weight for each template based on its utilization during tracking.
The template having the smallest weight is then replaced by the current tracking result.
However, this scheme cannot model the templates' representability
and cannot adapt to the degree of target's appearance changes,
thus lacks of discriminative power.
Our TRAC algorithm addresses both issues
and can robustly track targets
with appearance changes over time.




\subsection{Particle Filter}

The particle filter is widely used in visual tracking,
which is a combination of sequential importance sampling and resampling methods to solve the filtering problem.
It estimates the posterior distribution of state variables in a hidden Markov chain.
Let $\mathbf{s}_t$ and $\mathbf{y}_t$ denote the state variable at time $t$ and its observation respectively. The prediction of the state $\mathbf{s}_t$ given all previous observations up to time $t-1$ is given by
\begin{equation}
  p(\mathbf{s}_t|\mathbf{y}_{1:t-1}) = \int{p(\mathbf{s}_t|\mathbf{s}_{t-1})p(\mathbf{s}_{t-1}|\mathbf{y}_{1:t-1}) \, d\,\mathbf{s}_{t-1}}
\end{equation}
where $\mathbf{y}_{1:t-1} := (\mathbf{y}_1, \mathbf{y}_2, \cdots, \mathbf{y}_{t-1})$.
In the update step, the observation $\mathbf{y}_t$ is available, the state probability can be updated using the Bayes rule
\begin{equation}
  p(\mathbf{s}_t|\mathbf{y}_{1:t}) = \frac{p(\mathbf{y}_t|\mathbf{s}_t)p(\mathbf{s}_t|\mathbf{y}_{1:t-1})}{p(\mathbf{y}_t|\mathbf{y}_{t-1})}
\end{equation}
In the particle filter, the posterior $p(\mathbf{s}_t|\mathbf{y}_{1:t})$ is estimated by sequential importance sampling,
and we select an importance density $q(\mathbf{s}_{1:t}|\mathbf{y}_{1:t})$ such that $p(\mathbf{s}_{1:t},\mathbf{y}_{1:t}) = w_tq(\mathbf{s}_{1:t}|\mathbf{y}_{1:t})$
from which it is easy to draw samples,
where $q(\mathbf{s}_{1:t}|\mathbf{y}_{1:t}) = q(\mathbf{s}_{1:t-1}|\mathbf{y}_{1:t-1})q(\mathbf{s}_t|\mathbf{s}_{1:t-1},\mathbf{y}_t)$.
To generate $n$ independent samples (particles) $\{\mathbf{s}_1^i\}_{i=1}^n \sim q(\mathbf{s}_{1:t}|\mathbf{y}_{1:t})$ at time $t$, we generate $\mathbf{s}_1^i \sim q(\mathbf{s}_1|\mathbf{y}_1)$ at time 1, then $\mathbf{s}_k^i \sim q(\mathbf{s}_k|\mathbf{s}_{1:k}^i,\mathbf{y}_k)$ at time $k$, for $k = 2,\cdots, t$.
The weight of the particle $\mathbf{s}_t^i$ at time $t$, is updated as
\begin{equation}
  w_t^i = w_{t-1}^i\frac{p(\mathbf{y}_t|\mathbf{s}_t^i)p(\mathbf{s}_t^i|\mathbf{s}_{t-1}^i)}{q(\mathbf{s}_t^i|\mathbf{s}_{1:t-1}^i,\mathbf{y}_t)}
\end{equation}
At each time step, the particles are resampled according to their importance weights to generate new equally weighted particles. In order to minimize the variance of the importance weights at time $t$, the importance density is selected according to
$q(\mathbf{s}_t|\mathbf{s}_{1:t-1},\mathbf{y}_t) = p(\mathbf{s}_t|\mathbf{s}_{t-1},\mathbf{y}_t)$.

An affine motion model between consecutive frame is assumed in particle filters for visual tacking, as introduced in \cite{mei2011robust}.
That is, the state variable $\mathbf{s}_t$ is defined as a vector that consists of six parameters of the affine transformation, transforming the bounding box within each image frame to get an image patch of the target.
The state transition $p(\mathbf{s}_t|\mathbf{s}_{t-1})$ is defined as a multivariate Gaussian distribution with a different standard deviation for each affine parameter.
Since the velocity of the tracking target is unknown and can change during tracking,
it is modeled within the variance of the position parameters in the state transition.
In this way, the tracking techniques based on particle filters need a variety of state parameters,
which requires a large amount of particles to represent this distribution.
The observation $\mathbf{y}_t$ encodes the cropped region of interest by applying the affine transformation.
In practice, $\mathbf{y}_t$ is represented by the normalized features extracted from the region of interest.

\section{TRAC-Based Sparse Tracking}

\subsection{Sparse Tracking}

Under the tracking framework based on particle filtering,
the particles are randomly sampled around the current state of the target object according to $p(\mathbf{s}_t|\mathbf{s}_{t-1})$.
At time $t$, we consider $n$ particle samples $\{\mathbf{s}_t^i\}_{i=1}^n$,
which are sampled from the state of the previous resampled particles in time $t-1$,
 according to the predefined multivariate Gaussian distribution $p(\mathbf{s}_t|\mathbf{s}_{t-1})$.
The observations of these particles (\emph{i.e.}, the image features of the particles) in the $t$-th frame are denoted as $\mathbf{X} = [\mathbf{x}_1,\mathbf{x}_2,\cdots,\mathbf{x}_n] \in \Re^{d\times n}$,
where $\mathbf{x}_i$ represents the image features of the particle $\mathbf{s}_t^i$, and $d$ is the dimension of the feature.
In the noiseless case, each $\mathbf{x}_i$ approximately lies in a linear span of low-dimensional subspace,
which is encoded as a dictionary
$\mathbf{D} = [\mathbf{d}_1, \mathbf{d}_2, \cdots, \mathbf{d}_m] \in \Re^{d\times m}$ containing $m$  templates of the target,
such that $\mathbf{X} = \mathbf{D}\mathbf{Z}$,
where $\mathbf{Z} \in \Re^{m\times n}$ is a weight matrix of $\mathbf{X}$ with respect to $\mathbf{D}$.

When targets are partially occluded or corrupted by noise,
the negative effect can be modeled as sparse additive noise that can take a large value anywhere \cite{mei2011robust}.
To address this issue, the dictionary is augmented with trivial templates $\mathbf{I}_d = [\mathbf{i}_1, \mathbf{i}_2, \cdots, \mathbf{i}_d] \in \Re^{d\times d}$, where a trivial template $\mathbf{i}_i \in \Re^d$ is a vector with only one nonzero entry that can capture occlusion and pixel corruption at the $i$-th location:
\begin{equation}\label{Eq:model}
  \mathbf{X} = \left[\begin{array}{cc}
        \mathbf{D} & \mathbf{I}_d
      \end{array}\right]
  \left[ \begin{array}{c}
               \mathbf{Z} \\
               \mathbf{E}
             \end{array} \right]
   = \mathbf{BW}
\end{equation}
Because the particles $\{\mathbf{s}_t\}_{i=1}^n$ are represented by the corresponding image features $\{\mathbf{x}\}_{i=1}^n$, the observation probability $p(\mathbf{y}_t|\mathbf{s}_t^i)$ becomes $p(\mathbf{y}_t|\mathbf{x}_i)$,
which reflects the similarity between a particle and the templates.
The probability $p(\mathbf{y}_t|\mathbf{x}_i)$ is inversely proportional to the reconstruction error obtained by this linear representation.
\begin{equation}
\label{Eq:obs prob}
p(\mathbf{y}_t|\mathbf{s}_t^i) = \exp(-\gamma \|\mathbf{x}_i-\hat{\mathbf{x}}_i\|_2^2)
\end{equation}
where $\gamma$ is a predefined parameter and $\hat{\mathbf{x}}_i$ is the value of the particle representation predicted by Eq. \eqref{Eq:model}.
Then, the particle with the highest probability is selected as the object target at time $t$.

To integrate multimodal features in multi-task sparse tracking,
$n$ particles are jointly considered in estimating $\mathbf{W}$, 
and each particle has $K$ modalities of features. 
When multimodal features are applied,
the particle representation $\mathbf{X}$ can be denoted as $\mathbf{X} = \left[\mathbf{X}^1,\mathbf{X}^2,\cdots, \mathbf{X}^K\right]^{\top}$.
For each modality, the particle observation matrix $\mathbf{X}^k \in \Re^{d_k\times n}$ has $n$ columns
of normalized feature vectors for $n$ particles,
and $d_k$ is the dimensionality of the $k$-th modality such that $\sum_{k=1}^Kd_k = d$.
Then, the dictionary of the $k$-th modality is $\mathbf{B}^k = \left[ \mathbf{D}^k, \mathbf{I}_{d_k}\right]$,
thus Eq. \eqref{Eq:model} becomes
$\mathbf{X}_k = \mathbf{B}_k\mathbf{W}_k$.
The resulted representation coefficient matrix is a combination of all modality coefficients $\mathbf{W} = [\mathbf{W}^1, \mathbf{W}^2, \cdots, \mathbf{W}^K] \!\in\! \Re^{m\times (n\times K)}\!$.
In the multimodal sparse tracking framework, $\mathbf{W}$ is computed by:
\begin{eqnarray}\label{Eq:objOrigin}
  \min\limits_{\mathbf{W}} \sum_{k=1}^K \| \mathbf{B}^k \mathbf{W}^k - \mathbf{\mathbf{X}}^k \|_F^2
  + \lambda \| \mathbf{W} \|_{2,1}
\end{eqnarray}
where $\lambda$ is the trade-off parameter,
and the $\ell_{2,1}$ norm is denoted by $\| \mathbf{W} \|_{2,1} = \sum_i(\sqrt{\sum_jw_{i,j}^2})$
(with $w_{i,j}$ representing the element of the $i$-th row and $j$-th column in $\mathbf{W}$),
which enforces an $\ell_2$ norm on each row and an $\ell_1$ norm among rows, which introduces sparsity of the target templates.

\subsection{Temporal Consistency}

To robustly track deformable or changing objects in cluttered environments and address tracker drifting,
it is important to model the consistency of target appearances during a history of recent image frames.
While particle filters model the time propagation of each individual particle,
it cannot model the consistency of multiple particles.
In visual tracking, particles selected as the tracking results in multiple times are typically different
(especially when severe appearance change occurs),
which is critical but cannot be addressed by particle filters.
This shows, although the idea of temporal consistency is intuitive,
the solution is not obvious and heuristic.
In our TRAC algorithm, we propose a novel sparsity regularization to enforce temporal consistency.
Because the observation probability $p(\mathbf{y}_t|\mathbf{s}_t^i)$ is inversely proportional to the model error in Eq. \eqref{Eq:obs prob},
we enforce selecting the particles that are consistent with recently tracking results by applying temporal consistency in the objective function in Eq. \eqref{Eq:objOrigin}.

We denote $\mathbf{W}_t$ as the coefficient matrix of all particles with respect to $\mathbf{B}_t$
in the $t$-th frame, 
$\mathbf{w}_{t-l}$ is the coefficient vector of the tracking result
 (i.e., the selected particle encoding the target object) in the $(t-l)$-th frame with respect to $\mathbf{B}_t$,
and $\mathbf{W}_{t-l} = \mathbf{w}_{t-l}\mathbf{1}_n$ denotes the coefficient matrix for the target
with the same rank as $\mathbf{w}_{t-l}$.
Based on the insight that  a target object usually
has a higher similarity to a more recent tracking result
and this similarity decreases over time,
we employ a time decay factor to model the temporal correlation.
Then, the temporal consistency can be modeled using an autoregressive model as:
$\sum_{l=1}^T\alpha^l\| \mathbf{W}_t -\mathbf{W}_{t-l}\|_{2,1}$,
where $\alpha$ is the time decay parameter.
Thus, our multimodal sparse tracking task at time $t$ is formulated as:
\begin{eqnarray}\label{Eq:obj}
    & \min\limits_{\mathbf{W}_t} &\sum_{k=1}^K \| \mathbf{B}_t^k \mathbf{W}_t^k - \mathbf{\mathbf{X}}_t^k \|_F^2
  + \lambda_1 \| \mathbf{W}_t \|_{2,1} \nonumber \\
  && + \lambda_2 \sum_{l=1}^T \alpha^l  \|\mathbf{W}_t -\mathbf{W}_{t-l}\|_{2,1}
\end{eqnarray}
and $\mathbf{W}_{t-l}$ is computed by:
\begin{eqnarray}
\min\limits_{\mathbf{W}_{t-l}} & \sum_{k=1}^K \| \mathbf{B}_t^k \mathbf{W}_{t-l}^k - \mathbf{\mathbf{X}}_{t-l}^k \|_F^2
  + \lambda_1 \| \mathbf{W}_{t-l} \|_{2,1} \nonumber
\end{eqnarray}
The $i$-th row of the coefficient difference matrix $\mathbf{W}_t-\mathbf{W}_{t-l}$  in Eq. (\ref{Eq:obj})
denotes
the weight differences of the $i$-th template between the target in the $t$-th frame
and the previous tracking result in the $(t-l)$-th frame.
The $\ell_{2,1}$ norm of the coefficient difference $\|\mathbf{W}_t-\mathbf{W}_{t-l}\|_{2,1}$
enforces a small number of rows to have non-zero values,
i.e., only a small set of the templates can be different to represent the targets in frames $t$ and $t-l$.
In other words,
this regularization term encourages the target appearance in the current frame to be similar to the previous tracking results.
Thus, using this regularization,
the particles with appearances that are similar to the recently tracking results can be better modeled,
and the corresponding observation probability $p(\mathbf{y}_t|\mathbf{s}_t^i)$ is higher.
The particle with the highest observation probability in Eq. \eqref{Eq:obj} is then chosen as the tracking result.
When templates are updated (Sec. \ref{sec:sub:update}),
the coefficient matrices $\{\mathbf{W}_{t-l}\} (l=1,\dots,T)$ need to be recalculated.
If the tracking result in the frame $t-l$ is included in the current dictionary,
we don't use its coefficient to enforce consistency,
to avoid overfitting  (i.e., the dictionary can perfectly encode the tracking result at $t-l$ with no errors).



\subsection{Adaptive Template Update} \label{sec:sub:update}


The target appearance usually changes over time;
thus fixed templates typically cause the tracking drift problem.
To model the appearance variation of the target, the dictionary needs to be updated.
Previous techniques \cite{mei2011robust} for template update assign each template an importance weight
to prefer frequently used templates, 
and replace the template with the smallest weight by the current tracking result
if it is different form the highest weighted template.
However, these methods suffer from two key issues.
First, the update scheme does not consider the representability of these templates,
but only rely on their frequency of being used.
Thus, similar templates are usually included in the dictionary,
which decreases the discrimination power of the templates.
Second, previous update techniques are not adaptive;
they update the templates with the same frequency without modeling the target's changing speed.
Consequently, they are incapable of capturing the insight that when the target's appearance changes faster,
the templates must be updated more frequently, and vise versa.



To address these issues, we propose a novel adaptive template update scheme
that allows our TRAC algorithm to adaptively select target templates,
based on their representativeness and importance,
according to the degree of appearance changes during tracking.
When updating templates,
we consider their long-term-short-term representativeness.
The observation of recent tracking results are represented by
$\mathbf{Y} = \left[\mathbf{y}_t, \mathbf{y}_{t-1}, \cdots, \mathbf{y}_{t-(l-1)} \right] \in \Re^{d\times l}$,
where $\mathbf{y}_t$ is the observation (\emph{i.e.}, feature vector) of the particle chosen as the tracking target at time $t$,
which is used as the template candidate to update the dictionary $\mathbf{D} \in \Re^{d\times m}$.
Then, the objective is to select $r$ $(r < l, r < m) $ templates that are most representative in short-term from the recent tracking results,
which can be formulated to solve:
%
\begin{equation}\label{Eq:template}
  \min\limits_{\mathbf{U}}\|\mathbf{Y}-\mathbf{YU}\|_F^2 + \lambda_3 \| \mathbf{U}\|_{2,1}
\end{equation}
where $\mathbf{U}=[\mathbf{u}_1, \mathbf{u}_2, \cdots, \mathbf{u}_l] \in \Re^{l\times l}$,
and $\mathbf{u}_i$ is the weight of the template candidates to represent the  $i$-th candidate in $\mathbf{Y}$.
The $\ell_{2,1}$ norm enforces sparsity among the candidates,
which enables to select a small set of representative candidates.
After solving Eq. (\ref{Eq:template}),
we can sort the rows $\mathbf{U}^i$ ($i=1,\dots,l$) by the row-sum values of the absolute $\mathbf{U}$ in the
decreasing order,
resulting in a row-sorted matrix $\mathbf{U}'$.
A key contribution of our TRAC algorithm
is its capability to \emph{adaptively} select a number of templates,
which varies according to the degree of the target's appearance variation.
Given $\mathbf{U}'$, our algorithm determines the minimum $r$ value that satisfies $\frac{1}{l}\sum_{i=1}^{r} \|\mathbf{U}_i' \|_1 \geq \gamma$,
and selects the $r$ template candidates corresponding to the top $r$ rows of $\mathbf{U}'$,
where $\gamma$ is a threshold encoding our expect of the overall representativeness of the selected candidates
(e.g., $\gamma = 0.75$).
When the target's appearance remains the same in the recent tracking results,
one candidate will obtain a high row-sum value (while others have a value close to 0,
 due to the $\ell_{2,1}$ norm),
which will be selected as the single candidate.
On the other hand, when the target's appearance significantly changes,
since no single candidate can well represent others,
the rows of $\mathbf{U}$ will become less sparse and a set of candidates
can have a high row-sum value.
So, multiple candidates in the top rows of $\mathbf{U}'$ will be selected.
Therefore, our TRAC method is able to adaptively select a varying number of template candidates
based on their short-term representability, according to the degree of the target's appearance changes.

To update the dictionary $\mathbf{D}$,
the adaptively selected $r$ candidates are added to $\mathbf{D}$,
while the same number of templates must be removed from $\mathbf{D}$.
To select the templates to remove,
we compute the representativeness weight of the templates in $\mathbf{D}$,
using the same formulation in Eq. (\ref{Eq:template}).
Since the dictionary incorporates template information from the beginning of tracking,
we call the weight the long-term representativeness.
Then, the templates to remove from $\mathbf{D}$ are selected according to a combined weight:
\begin{equation}
  \mathbf{w} = \beta \mathbf{w}_{rep} + (1-\beta) \mathbf{w}_{imp}
\end{equation}
where
$\mathbf{w}_{rep}$ denotes the normalized long-term representativeness weight,
$\mathbf{w}_{imp}$ denotes the traditional normalized importance weight,
and $\beta$ is a trade-off parameter.
The $r$ templates in $\mathbf{D}$ with the minimum weights are removed.


\subsection{Optimization Algorithm}

Although the optimization problems in Eqs. (\ref{Eq:obj}) and (\ref{Eq:template}) are convex,
since their objective function contains non-smooth terms,
they are still challenging to solve.
We introduce a new efficient algorithm to solve both problems,
and provide a theoretical analysis to prove that the algorithm converges to the global optimal solution.
Since Eq. (\ref{Eq:template}) is a special case of Eq. (\ref{Eq:obj}) when $\lambda_2 = 0$,
we derive the solution according to the notation used in Eq. (\ref{Eq:obj}).
For a given matrix $\mathbf{W} = [w_{i,j}]$, we represent its $i$th row as $\mathbf{w}^i$ and the $j$th column as $\mathbf{w}_j$.
Given $\mathbf{W}_t^k = [\mathbf{w}_{t1}^k, \mathbf{w}_{t2}^k, \cdots, \mathbf{w}_{tn}^k]$, taking the derivative of the objective with respect to $\mathbf{W}_t^k (1\leq k \leq K)$, and setting it to zero, we obtain
\begin{eqnarray}
  (\mathbf{B}_t^k)^\top \mathbf{B}_t^k \mathbf{W}_t^k - (\mathbf{B}_t^k)^\top \mathbf{X}_t^k + \lambda_1\tilde{\mathbf{D}}\mathbf{W}_t^k   \nonumber\\
  + \lambda_2\sum_{l=1}^T \alpha^l \mathbf{D}^l(\mathbf{W}_t^k - \mathbf{W}_{t-l}^k) = 0
\end{eqnarray}
where $\mathbf{W}_{t-l}^k$ is the coefficient of the $k$th view in the tracking result at time $t-l$,
$\tilde{\mathbf{D}}$ is a diagonal matrix with the $i$th diagonal element as $\frac{1}{2\|\mathbf{w}_t^i\|_2}$,
and $\mathbf{D}^l$ is a diagonal matrix with the $i$th diagonal matrix as $\frac{1}{2\|\mathbf{w}_t^i-\mathbf{w}_{t-l}^i \|_2}$. Thus we have:
\begin{eqnarray}\label{Eq:updateW}
  \mathbf{W}_t^k &=& \left( (\mathbf{B}_t^k)^\top \mathbf{B}_t^k + \lambda_1\tilde{\mathbf{D}} + \lambda_2\sum_{l=1}^T\alpha^lD^l \right)^{-1}    \nonumber\\
  && \cdot \left( (\mathbf{B}_t^k)^\top \mathbf{X}_t^k + \lambda_2\sum_{l=1}^T\alpha^l\mathbf{D}^l\mathbf{W}_{t-l}^k \right)
\end{eqnarray}
Note that $\tilde{\mathbf{D}}$ and $\mathbf{D}^l(1\leq l\leq T)$ are dependent on $\mathbf{W}_t$ and thus are also unknown variables. We propose an iterative algorithm to solve this problem described in Algorithm 1.

\textbf{Convergence analysis.} The following theorem guarantees the convergence of Algorithm \ref{alg:optimization}.

\newtheorem{theorem}{Theorem}
\begin{theorem}
Algorithm \ref{alg:optimization} decreases the objective value of Eq. (\ref{Eq:obj}) in each iteration.
\end{theorem}
\begin{proof}
\small
In each iteration of Algorithm \ref{alg:optimization}, according to Step 3 to 5, we know that
\begin{eqnarray*}
  (\mathbf{W}_t)_{s+1}\!\!\!\!  &=&\!\!\!\! \min\limits_{\mathbf{W}_t} \sum_{k=1}^K \| \mathbf{B}_t^k\mathbf{W}_t^k - \mathbf{X}_t^k \|_F^2 + \lambda_1\Tr{\mathbf{W}_t^\top\tilde{\mathbf{D}}_{s+1}\mathbf{W}_t}  \nonumber\\
    &+& \!\!\!\!  \lambda_2\sum_{l=1}^T\Tr{(\mathbf{W}_t-\mathbf{W}_{t-l})^\top \mathbf{D}_{s+1}^l (\mathbf{W}_t-\mathbf{W}_{t-l})}
\end{eqnarray*}
Thus, we can derive:
\begin{eqnarray*}
 \sum_{k=1}^K \| \mathbf{B}_t^k(\mathbf{W}_t^k)_{s+1} - \mathbf{X}_t^k \|_F^2 + \lambda_1\Tr{(\mathbf{W}_t)_{s+1}^\top\tilde{\mathbf{D}}_{s+1}(\mathbf{W}_t)_{s+1}}  \nonumber\\
+ \lambda_2 \sum_{l=1}^T \alpha^l \Tr{\left( (\mathbf{W}_t)_{s+1}-\mathbf{W}_{t-l}\right)^\top \mathbf{D}^l_{s+1} \left( (\mathbf{W}_t)_{s+1}-\mathbf{W}_{t-l} \right)}  \nonumber\\
\leq \sum_{k=1}^K \| \mathbf{B}_t^k(\mathbf{W}_t^k)_{s} - \mathbf{X}_t^k \|_F^2 + \lambda_1\Tr{(\mathbf{W}_t)_{s}^\top\tilde{\mathbf{D}}_{s+1}(\mathbf{W}_t)_{s}}  \nonumber\\
+ \lambda_2 \sum_{l=1}^T\alpha^l \Tr{( (\mathbf{W}_t)_{s}-\mathbf{W}_{t-l})^\top \mathbf{D}^l_{s+1} ( (\mathbf{W}_t)_{s}-\mathbf{W}_{t-l})}
\end{eqnarray*}
Substituting $\tilde{\mathbf{D}}$ and $\mathbf{D}^l$ by definitions, we obtain:
\begin{eqnarray}\label{Eq:less1}
\mathcal{L}_{s+1} + \lambda_1\sum_{i=1}^m\frac{\|(\mathbf{w}_t^i)_{s+1}\|_2^2}{2\|(\mathbf{w}_t^i)_s\|_2} +
   \lambda_2\sum_{l=1}^T \alpha^l\sum_{i=1}^m\frac{\| (\mathbf{w}_t^i)_{s+1}-\mathbf{w}_{t-l}^i \|_2^2}{2\| (\mathbf{w}_t^i)_s - \mathbf{w}_{t-l}^i \|_2}    \nonumber\\
\leq \mathcal{L}_s + \lambda_1\sum_{i=1}^m\frac{\|(\mathbf{w}_t^i)_s\|_2^2}{2\|(\mathbf{w}_t^i)_s\|_2} +
   \lambda_2\sum_{l=1}^T \alpha^l\sum_{i=1}^m\frac{\| (\mathbf{w}_t^i)_s-\mathbf{w}_{t-l}^i \|_2^2}{2\| (\mathbf{w}_t^i)_s - \mathbf{w}_{t-l}^i \|_2} \nonumber
\end{eqnarray}
where $\mathcal{L}_{s} = \sum_{k=1}^K \| \mathbf{B}_t^k(\mathbf{W}_t^k)_{s} - \mathbf{X}_t^k \|_F^2$.
Since it can be easily verified that for the function $f(x) = x-\frac{x^2}{2\alpha}$,
given any $x\neq\alpha \in \Re, f(x)\leq f(\alpha)$ holds, we can derive:
\begin{eqnarray}\label{Eq:less2}
&& \sum_{i=1}^m\| (\mathbf{w}_t^i)_{s+1} \|_2 - \sum_{i=1}^m\frac{\|(\mathbf{w}_t^i)_{s+1}\|_2^2}{2\| (\mathbf{w}_t^i)_s \|_2}     \nonumber\\
&& \leq \sum_{i=1}^m\|(\mathbf{w}_t^i)_s \|_2 - \sum_{i=1}^m\frac{\|(\mathbf{w}_t^i)_{s}\|_2^2}{2\| (\mathbf{w}_t^i)_s \|_2} \nonumber
\end{eqnarray}
and
\begin{eqnarray}\label{Eq:less3}
&& \sum_{i=1}^m\| (\mathbf{w}_t^i)_{s+1}-\mathbf{w}_{t-l}^i \|_2 - \sum_{i=1}^m\frac{\|(\mathbf{w}_t^i)_{s+1}-\mathbf{w}_{t-l}^i\|_2^2}{2\| (\mathbf{w}_t^i)_s-\mathbf{w}_{t-l}^i \|_2} \leq    \nonumber\\
&& \sum_{i=1}^m\|(\mathbf{w}_t^i)_s-\mathbf{w}_{t-l}^i \|_2 - \sum_{i=1}^m\frac{\|(\mathbf{w}_t^i)_{s}-\mathbf{w}_{t-l}^i\|_2^2}{2\| (\mathbf{w}_t^i)_s-\mathbf{w}_{t-l}^i \|_2}
\end{eqnarray}
Adding the previous three equations
on both sides (note Eq. (\ref{Eq:less3}) is repeated for $1\leq l\leq T$), we have
\begin{eqnarray}
\mathcal{L}_{s+1} + \lambda_1\sum_{i=1}^m\|  (\mathbf{w}_t^i)_{s+1}\|_2 + \lambda_2\sum_{l=1}^T\alpha^l\sum_{i=1}^m\| (\mathbf{w}_t^i)_{s+1}-\mathbf{w}_{t-l}^i \|_2   \nonumber\\
\leq \mathcal{L}_s + \lambda_1\sum_{i=1}^m\| (\mathbf{w}_t^i)_{s}\|_2 + \lambda_2\sum_{l=1}^T\alpha^l\sum_{i=1}^m\| (\mathbf{w}_t^i)_{s}-\mathbf{w}_{t-l}^i \|_2 \nonumber
\end{eqnarray}
Therefore, the algorithm decreases the objective value in each iteration.
Since the problem in Eq. (\ref{Eq:obj}) is convex,
the algorithm converges to the global solution.
\end{proof}


\begin{algorithm}[t]

{
\caption{An efficient iterative algorithm to solve the optimization problems in Eqs. (\ref{Eq:obj}) and (\ref{Eq:template}).}
\label{alg:optimization}

    \SetKwInOut{Input}{Input}
    \SetKwInOut{Output}{Output}
    \Input{$\mathbf{B}_t, \mathbf{X}_t$}
        \Output{$\left( \mathbf{W}_t\right)_s \in \Re^{m\times (nK)}$}

    Let $s = 1$. Initial $(\mathbf{W}_t)_s$ by solving $\min\limits_{\mathbf{W}_t} \sum_{k=1}^K \| \mathbf{B}_t^k\mathbf{W}_t^k - \mathbf{X}_t^k \|_F^2$. \\
    \While{not converge} {
    Calculate the diagonal matrix $\tilde{\mathbf{D}}_{s+1}$, where the $i$th diagonal element is $\frac{1}{2\|(\mathbf{w}_t^i)_s\|_2}$.

    Calculate the diagonal matrices $\mathbf{D}^l_{s+1} (1\leq l \leq T)$, where the $i$th diagonal element is $\frac{1}{2\|(\mathbf{w}_t^i)_s-\mathbf{w}_{t-l}^i \|_2}$.

    For each $\mathbf{W}_t^k(1\leq k\leq K)$, calculate $(\mathbf{W}_t^k)_{s+1}$ using Eq. (\ref{Eq:updateW}).

    s = s+1
    }
}

\end{algorithm}

\begin{figure*}[tbh]
    \includegraphics[width=0.995\textwidth]{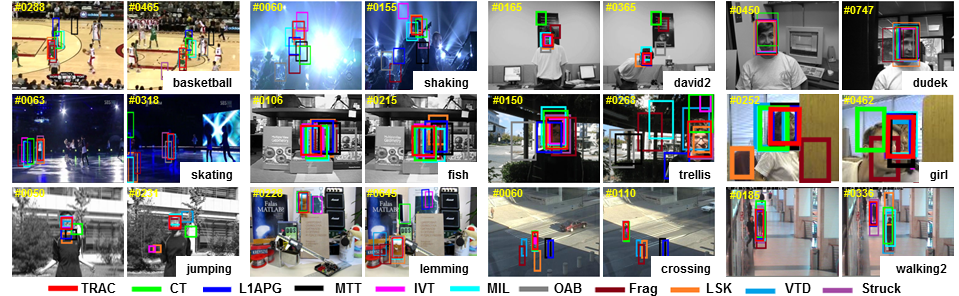}
    \centering
    \caption{Tracking results of 11 trackers (denoted in different colors) on 12 image sequences.
    Frame indices are shown in the top left corner in yellow colors.
    Results are best viewed in color on high-resolution displays.}\label{Fig:qualitative}
\end{figure*}

\section{Experiments}
To evaluate the performance of the proposed TRAC method,
we performed  extensively validation on twelve challenging image sequences that are publicly available
from the widely used
Visual Tracker Benchmark dataset \cite{wu2013online}\footnote{
The Visual Tracker Benchmark: www.visual-tracking.net.}.
The used image sequences contain a variety of target objects under static or dynamic background.
The length of the image sequences is also varied with the shortest under 100 frames and the longest over 1000 frames.
Each frame of the sequence is manually annotated with the corresponding ground-truth bounding box for the tracking target;
the attributes and challenges of each sequence that may affect tracking performance are also provided in the dataset.

Throughout the experiments, we employed the parameter set of $\lambda_1 = 0.5$, $\lambda_2 = 0.1$, $\lambda_3 = 0.5$, $\alpha = 0.1$, $\beta = 0.5$,
$n = 400$, 
and
$m = 10$.  
To represent the tracking targets,
we employed four popular visual features that were widely used in previous sparse tracking methods:
color histograms, intensity, histograms of oriented gradients (HOG), and local binary patterns (LBP).
We compared our TRAC algorithm with ten state-of-the-art methods,
including trackers based on
\underline{(1)} multiple instance learning (MIL) \cite{babenko2009visual},
\underline{(2)} online Adaboost boosting (OAB) \cite{grabner2006real},
\underline{(3)} L1 accelerated proximal gradient tracker (L1APG) \cite{bao2012real},
\underline{(4)} Struck \cite{hare2011struck},
\underline{(5)} circulant structure tracking with kernels (CSK) \cite{henriques2012exploiting},
\underline{(6)} local sparse and K-selection tracking (LSK) \cite{liu2011robust},
\underline{(7)} multi-task tracking (MTT) \cite{zhang2012robust},
\underline{(8)} incremental visual tracking (IVT) \cite{ross2008incremental},
\underline{(9)} fragments-based tracking (Frag) \cite{adam2006robust}, and
\underline{(10)} visual tracking decomposition (VTD) \cite{kwon2010visual}.

\subsection{Qualitative Evaluation}

The qualitative tracking results obtained by our TRAC algorithm is  shown in  Figure \ref{Fig:qualitative}.
We analyze and compare the performance when various challenges are present,
as follows.


\textbf{Occlusion}:
The \textit{walking2} and \textit{girl} sequences track a person body or a human face while occluded by another person.
In the \textit{walking2} sequence, the OAB, Frag, MIL, CT, LSK, and VTD methods fail when the walking woman is occluded by a man.
The Struck method shows more tracking errors from the accurate position.
On the other hand, TRAC, L1APG, MTT, and IVT methods successfully track the target throughout the entire sequence.
The main challenge of the \textit{girl} sequence is occlusion and pose variation.
Frag fails when the girl starts to rotate;
LSK fails when the girl completely turns her back towards the camera.
The IVT method fails around frame 125 when the girl keeps rotating,
and the CT and MIL methods experience significant drift at the same time.
When the man's face occludes the girl,
the VTD method starts to track the men but comes back to the target when the man disappears.
The TRAC, L1APG, MTT, OAB, and Struck methods accurately track the target face in the entire sequence.


\textbf{Background Clutter}:
The \textit{basketball} and \textit{skating1} sequences track a fast moving human among other people,
with significant background clutter, occlusion and deformation.
In the \textit{basketball} sequence,
the TRAC, VTD, and Frag methods track the correct target throughout the entire sequence,
while Frag suffers more errors from the accurate position.
Other trackers fail to track the target at different time frames.
Due to enforcing temporal consistency and adaptively updating templates,
our TRAC method accurately tracks the fast moving human body.
In the \textit{skating1} sequence, the TRAC and VTD methods can track the target most of the time.
The LSK and OAB trackers can keep tracking most of the time
but significantly drift away at the frames where the background is dark.
Struck fails when the target is occluded by another person.
Other trackers fail at earlier time frames due to the target or background motion.



\textbf{Illumination Variation}: The main challenge of the \textit{shaking} and \textit{fish} sequences is illumination change.
In \textit{shaking},
the OAB, CT, IVY, Frag and MTT trackers fail to track the target face in frames around 17, 21, 25, 53, 60, respectively.
Struck cannot track the accurate position most of the time and drift far away.
LSK fails in frame 18 but recovers in frame 59; 
it also suffers tracking drift when the hat occludes the man's face.
In contrast, TRAC and VTD successfully track the target for the whole video.
In the \textit{fish} sequence, OAB and LSK fail in frames 25 and 225, respectively.
L1APG, MTT, Frag, MIL, and VTD track part of the target but gradually drift away.
The TRAC, IVT, Struck, and CT methods accurately track the entire sequence despite large illumination changes,
while CT is less accurate compared to other successful methods.

\begin{figure}[htb]
\vspace{-10pt}
  \subfigure[precision]{
    \begin{minipage}[b]{0.234\textwidth}
      \centering
        \includegraphics[height=1.25in]{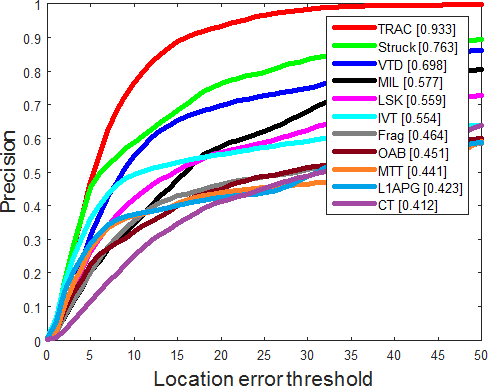}
    \end{minipage}}
  \subfigure[success rate]{
    \begin{minipage}[b]{0.23\textwidth}
      \centering
        \includegraphics[height=1.25in]{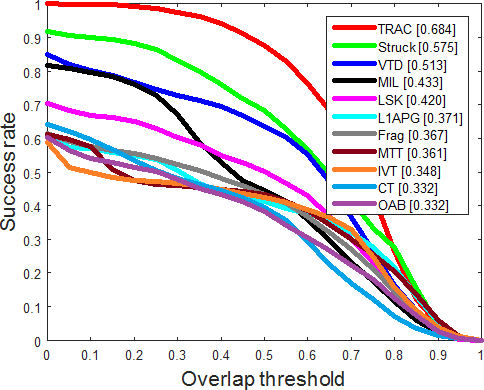}
    \end{minipage}}
  \caption{Overall tracking performance of our TRAC algorithm and comparison with previous state-of-the-art methods.}
  \label{Fig:quantitative_all}
\end{figure}

\begin{figure*}[t]
 \hspace{-3pt}
  \subfigure{
    \begin{minipage}[b]{0.23\textwidth}
      \centering
        \includegraphics[height=1.34in]{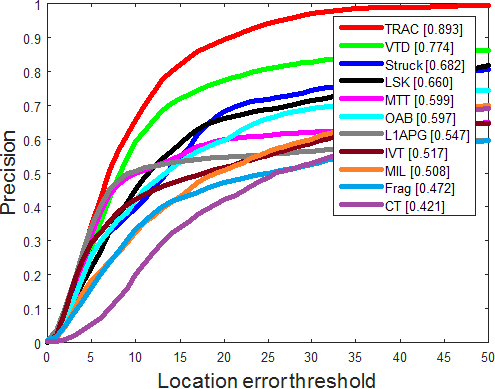}
    \end{minipage}}
  \hspace{3pt}
  \subfigure{
    \begin{minipage}[b]{0.23\textwidth}
      \centering
        \includegraphics[height=1.34in]{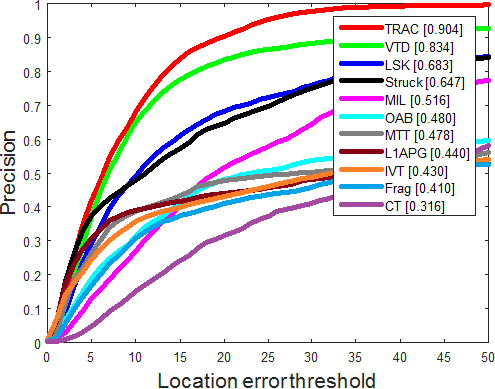}
    \end{minipage}}
  \hspace{3pt}
  \subfigure{
    \begin{minipage}[b]{0.23\textwidth}
      \centering
        \includegraphics[height=1.34in]{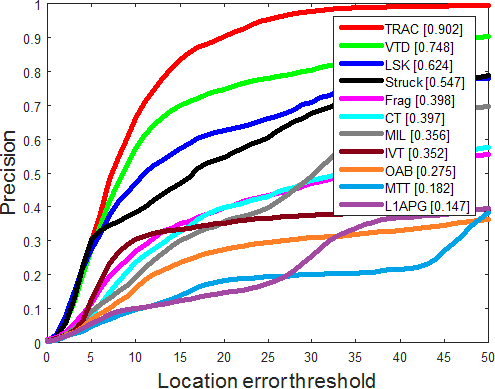}
    \end{minipage}}
  \hspace{3pt}
  \subfigure{
    \begin{minipage}[b]{0.23\textwidth}
      \centering
        \includegraphics[height=1.34in]{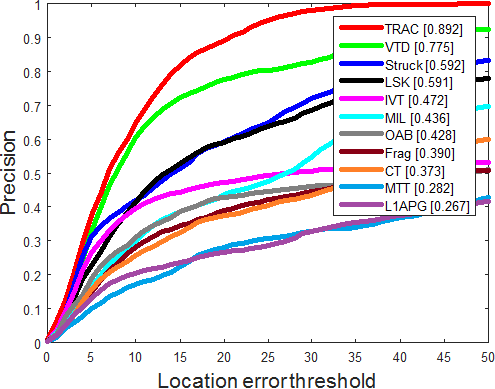}
    \end{minipage}}

  \setcounter{subfigure}{0}
  \vspace{-5pt}
  \hspace{-3pt}
  \subfigure[occlusion]{
    \begin{minipage}[b]{0.23\textwidth}
      \centering
        \includegraphics[height=1.34in]{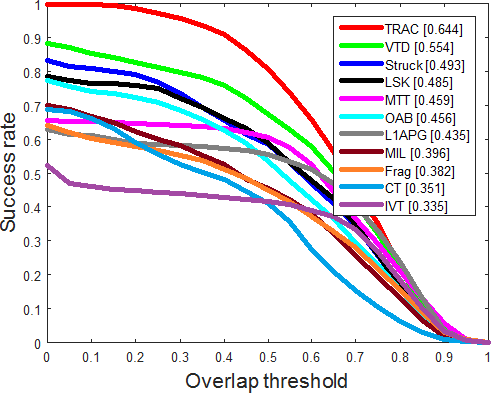}
    \end{minipage}}
  \hspace{3pt}
  \subfigure[rotation]{
    \begin{minipage}[b]{0.23\textwidth}
      \centering
        \includegraphics[height=1.34in]{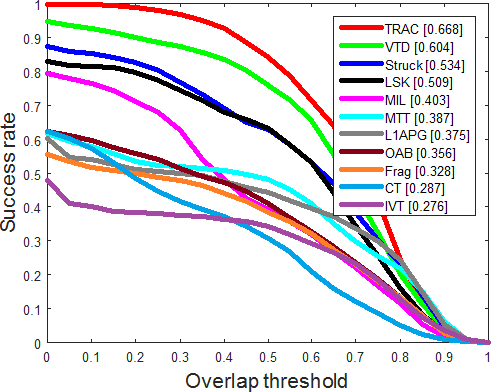}
    \end{minipage}}
  \hspace{3pt}
  \subfigure[illumination variation]{
    \begin{minipage}[b]{0.23\textwidth}
      \centering
        \includegraphics[height=1.34in]{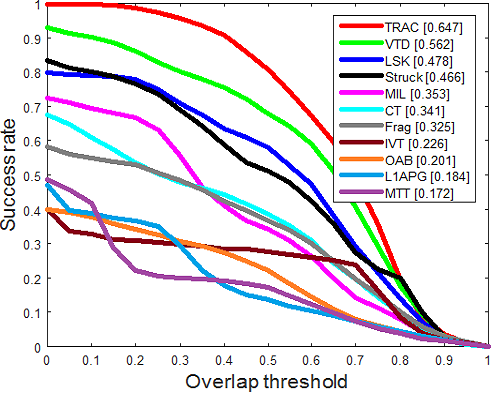}
    \end{minipage}}
  \hspace{3pt}
  \subfigure[background clutter]{
    \begin{minipage}[b]{0.23\textwidth}
      \centering
        \includegraphics[height=1.34in]{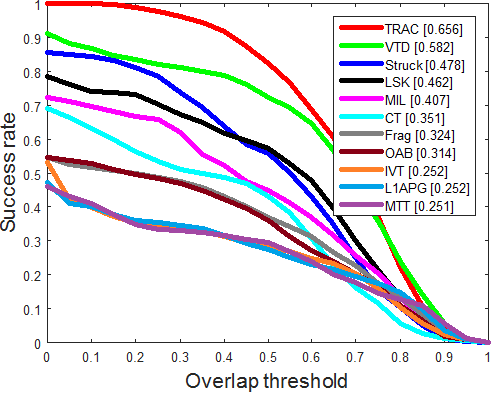}
    \end{minipage}}
  \caption{Precision and success plots evaluated on image sequences with the challenges of (a) occlusion, (b) rotation (including in-plane and out-of-plane rotation), (c) illumination variation, and  (d) background clutter.}
  \label{fig:attributes} 
\end{figure*}


\textbf{Pose Variation}: The \textit{david2}, \textit{dudek}, and \textit{trellis} sequences track human faces
in different situations with significant pose changes.
In \textit{david2},
CT fails at the very beginning; Frag fails around frame 165;
OAB and LSK start to drift at frame 159 and 341, respectively, and then fail.
MIL roughly tracks the target but exhibits significant drifts.
In the \textit{dudek} sequence, occlusion of hands occurs at frame 205,
where the CT, OAB methods start to drift shortly after.
The Frag approach suffers more drifts than other trackers when pose changes,
and fails around frame 906.
The OAB method fails around frame 975, when the target is partially out of view.
 The L1APG method experiences significant drift at frame 1001 and keeps drifting from the accurate position to the end of the sequence.
 In the \textit{trellis} sequence, the OAB, MTT, IVT, Frag, L1APG, MIL, CT, VTD methods fail around frames 115, 192, 210, 212, 239, 240, 321, 332, respectively.
Struck successfully tracks the moving faces with slight tracking drifts.
Our TRAC tracker accurately tracks the moving targets with significant pose changes
in all three videos,
due to its ability to adaptively update templates and enforce temporal consistence.

%

\subsection{Quantitative Evaluation}

We also quantitatively evaluate our TRAC method's performance using the precision and success rate \cite{wu2013online}.
The precision metric is computed using the center location error,
which is the Euclidean distance between the center of the tracked target and the ground truth in each frame.
The plot is generated as the percentage of frames whose center location error is within the given threshold versus the predefined threshold.
The representative precision score is calculated with the threshold set to 20 pixels.
The metric of success rate is used to evaluate the bounding box overlap.
The overlap score is defined as the Jaccard similarity:
Given the tracked bounding box $ROI_T$ and the ground truth bounding box $ROI_G$,
it is calculated by $s = \frac{|ROI_T\bigcap ROI_G|}{|ROI_T\bigcup ROI_G|}$.
The success plot is generated as the ratio of successful frames at the threshold versus the predefined overlap score threshold ranging from 0 to 1.

To quantitatively analyze our algorithm's performance and compare with other methods,
we compute the average frame ratio for the center location error and the bounding box overlap score,
using the 12 image sequences.
The overall performance is demonstrated in Figure \ref{Fig:quantitative_all}.
The results show that our TRAC algorithm achieves the state-of-the-art tracking performance,
and significantly outperforms the previous 10 methods on all image sequences.
To evaluate the robustness of the proposed tracker in different challenging conditions,
we evaluate the performance according to the attributes provided by the image sequences,
including
occlusion, rotation, illumination variation, and background clutter.
As illustrated by the results in Figure \ref{fig:attributes},
our TRAC algorithm performs significantly better than previous methods,
which validates the benefit of enforcing temporal consistency
and adaptively updating target templates.



\section{Conclusion}
In this paper, we introduce a novel sparse tracking algorithm
that is able to model the temporal consistency of the targets
and adaptively update the templates based on their long-term-short-term representability.
By introducing a novel structured norm as a temporal regularization,
our TRAC algorithm can effectively enforce temporal consistency,
thus alleviating the issue of tracking drifting.
The proposed template update strategy considers the long-term-short-term
representability of the target templates
and is capable of selecting an adaptive number of templates, which varies according to the degree of the tracking target's appearance variations.
This strategy makes our approach highly robust to the target's appearance changes
due to occlusion, deformation, and pose changes.
Both abilities are achieved via structured sparsity-inducing norms,
and tracking is performed using particle filters.
To solve the formulated sparse tracking problem,
we implement a new optimization solver that offers a theoretical guarantee to efficiently find the optimal solution.
Extensive empirical studies have been conducted using the Visual Tracker Benchmark dataset.
The qualitative and quantitative results have validated that
our TRAC approach obtains very promising visual tracking performance,
and significantly outperforms the previous state-of-the-art techniques.
The proposed strategies not only address the visual tracking task,
but also can benefit addressing a wide range of problems
involving smooth temporal sequence modeling in artificial intelligence.



%


\bibliographystyle{named}
\bibliography{ijcai16}

\end{document}